\PassOptionsToPackage{table}{xcolor}
\documentclass[sigconf]{aamas}
\makeatletter
\def\aamas@legacy@editors{\fi}
\ifx\acmConference@editors\aamas@legacy@editors
  \expandafter\@firstofone
\else
  \expandafter\@gobble
\fi
{\acmConference@editors\gdef\acmConference@editors{}}
\let\aamas@legacy@editors\@undefined
\makeatother

\usepackage[T1]{fontenc}
\usepackage[utf8]{inputenc}
\usepackage{microtype}
\usepackage{booktabs}
\usepackage{multirow}
\usepackage{graphicx}
\usepackage{xcolor}
\usepackage[table]{xcolor}
\usepackage{array}
\usepackage{tabularx}
\usepackage{enumitem}
\usepackage{cleveref}
\usepackage{algorithm}
\usepackage{algpseudocode}
\usepackage{bm}
\usepackage{pgfplots}
\usepackage{url}
\pgfplotsset{compat=1.16}

\newtheorem{theorem}{Theorem}

\newtheorem{definition}{Definition}
\newtheorem{assumption}{Assumption}

\AtEndPreamble{%
  \hypersetup{
    pdfauthor={},
    colorlinks=true,
    linkcolor=blue,
    citecolor=black,
    urlcolor=blue,
    pdfborder={0 0 0},
    pdfborderstyle={}
  }%
}

\newsavebox{\ManuscriptLinkBox}
\DeclareRobustCommand{\ManuscriptLinkFrame}[2]{%
  \begingroup
    \setlength{\fboxsep}{0.3pt}%
    \setlength{\fboxrule}{0.6pt}%
    \sbox{\ManuscriptLinkBox}{#2}%
    \leavevmode
    \rlap{%
      \kern-\dimexpr\fboxsep+\fboxrule\relax
      \smash{\textcolor{#1}{\fbox{\phantom{\usebox{\ManuscriptLinkBox}}}}}%
    }%
    \usebox{\ManuscriptLinkBox}%
  \endgroup
}
\AddToHook{begindocument/end}{%
  \let\ManuscriptOriginalCiteNumFont\citenumfont
  \renewcommand*{\citenumfont}[1]{%
    \ManuscriptLinkFrame{green}{\ManuscriptOriginalCiteNumFont{#1}}%
  }%
  \let\ManuscriptOriginalEqref\eqref
  \renewcommand*{\eqref}[1]{%
    \begingroup
      \hypersetup{linkcolor=black}%
      \ManuscriptLinkFrame{red}{\ManuscriptOriginalEqref{#1}}%
    \endgroup
  }%
}

\newcommand{\Bop}{\mathbf{B}}
\newcommand{\rhoB}{\rho(\Bop)}
\newcommand{\phiH}{\Phi_H}
\definecolor{wmsarblue}{RGB}{0,105,180}
\definecolor{engred}{RGB}{180,50,30}
\newcommand{\method}{\textsc{ReCore}}
\makeatletter
\gdef\@copyrightpermission{
  \begin{minipage}{0.2\columnwidth}
   \href{https://creativecommons.org/licenses/by/4.0/}{\includegraphics[width=0.90\textwidth]{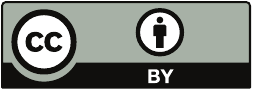}}
  \end{minipage}\hfill
  \begin{minipage}{0.8\columnwidth}
   \href{https://creativecommons.org/licenses/by/4.0/}{This work is licensed under a Creative Commons Attribution International 4.0 License.}
  \end{minipage}
  \vspace{5pt}
}
\makeatother

\setcopyright{ifaamas}
\acmConference[AAMAS 2027]{Proc.\@ of the 26th International Conference
on Autonomous Agents and Multiagent Systems (AAMAS 2027)}{3--7 May 2027}{Hanoi, Vietnam}{M.~Baldoni, F.~Fang, W.~Yeoh, N.~Yorke-Smith (eds.)}
\copyrightyear{2027}
\acmYear{2027}
\acmDOI{}
\acmISBN{}
\acmSubmissionID{}
\submissionType{Research Paper Track}

\title{Repair the Core, Not the Symptom: \\Planning-Graph Repair for Long-Horizon Agents}
\author{Xinyuan Song}
\affiliation{
  \institution{Emory University}
  \city{Atlanta}
  \country{USA}
}
\email{xinyuan.song@emory.edu}
\author{Zekun Cai}
\affiliation{
  \institution{The University of Tokyo}
  \city{Tokyo}
  \country{Japan}
}
\email{caizekun@csis.u-tokyo.ac.jp}

\makeatletter
\newcommand{\reviewcodeurl}{\if@ACM@anonymous\texttt{[URL omitted for anonymous review]}\else\url{https://github.com/Hik289/world-model-corrector.git}\fi}
\makeatother

\begin{abstract}
Long-horizon language agents make interdependent decisions across multiple rounds, forming growing planning graphs in which later decisions depend on earlier states, observations, and tool outputs. Errors in these graphs can therefore propagate through subsequent decisions and amplify over long rollouts. Correcting the most visible failure may therefore fix a symptom while leaving the core that drives downstream errors unchanged. Replaying the entire planning graph provides the opposite extreme, but becomes increasingly expensive as the graph grows and forces an LLM to reason over substantial irrelevant context. We study \emph{planning-graph repair}: given a failed rollout, the goal is to identify and correct the compact core responsible for continued error propagation while preserving the remainder of the graph. We first evaluate a range of practical engineering repair strategies, including single-error and TopK repair, temporal windows, neighborhood expansion, cascade-aware repair, and full-graph LLM correction. We then propose \textbf{\method{}} (\textbf{Repair the Core}), which models the coupled propagation of node and edge errors to identify the core of a failed planning graph. Rather than repairing the largest observed errors or expanding a fixed local neighborhood, \method{} selects a connected core whose correction maximally reduces residual error amplification, and exposes only this core to the LLM for diagnosis and repair. We also theoretically establish a bound on long-horizon planning regret. Extensive experiments on synthetic Agent Calling-Tree graphs and the Who\&When attribution benchmark show that \method{} improves long-horizon stabilization and root-cause localization by targeting high-amplification repair regions, matching full-context recovery with up to $7.4\times$ fewer tokens for every successful repair. Our code is available at: \url{https://github.com/Hik289/world-model-corrector.git}.
\end{abstract}


\begin{document}
\emergencystretch 3em
\maketitle


\section{Introduction}
\label{sec:intro}

Long-horizon language agents increasingly perform complex tasks that require extended sequences of interdependent decisions, such as navigating realistic websites \citep{webarena} and executing multi-step workflows across interactive environments \citep{wang2025agentworkflow}. In these tasks, each decision can use the results of multiple earlier steps, while the same intermediate result may influence several later actions. These dependencies naturally form a growing planning graph, where nodes represent intermediate states or decisions and edges capture how earlier results affect subsequent planning. A central challenge is that errors in long-horizon agents can propagate: early mistakes can affect subsequent reasoning and actions \citep{agenterrorbench}, while accumulating planning dependencies can make these errors increasingly difficult to correct \citep{wang2026reasoningfails}. An incorrect intermediate state can therefore affect multiple downstream decisions, allowing errors to propagate and amplify throughout the planning graph. Consequently, the state where a failure is most visible may not be the core responsible for producing it.

This creates a fundamental challenge for repairing failed agent plans. The state exhibiting the largest observed error need not be the source that drives subsequent failures: later errors may instead arise because they depend on an earlier corrupted state \citep{rafi2026falat}. Identifying the decisive step behind a failed trajectory is itself a difficult failure-attribution problem \citep{zhang2025whowhen}. Correcting only the most visible error may therefore remove a symptom while leaving the underlying error-propagation structure intact. Replaying the entire planning graph provides the opposite extreme, but becomes increasingly expensive as the graph grows, requiring additional computation and tool calls while forcing an LLM to identify a small relevant cause within substantial unrelated context. Long-context models can themselves struggle to reliably use relevant information embedded within large inputs \citep{liu2024lostmiddle}. Effective planning-graph repair therefore requires identifying the compact core whose correction prevents errors from continuing to propagate, rather than simply repairing the most visible failure or reconstructing the entire graph.

We first investigate several practical engineering approaches to this problem. Pointwise repair selects the state with the largest observed error, while TopK repair extends this strategy to several suspicious states. Temporal-window methods repair nearby steps, neighborhood methods expand around suspicious nodes, and cascade-aware methods trace observed downstream failures. At the other extreme, full-graph LLM correction provides the complete planning graph to the LLM for root-cause analysis and repair. These strategies connect naturally to feedback, verification, self-refinement, and replanning methods \citep{alfworld-repair,llm-repair-plan,reflexion2023,lats2024,cobbe2021verifiers,lightman2023verify}. Our evaluation shows the central difficulty: strategies based on visible error, temporal proximity, or graph distance can miss the core that drives downstream propagation, whereas full-graph correction incurs substantially greater context and computation.

We therefore propose \textbf{\method{}} (\textbf{Repair the Core}), a planning-graph repair method that identifies and corrects the core responsible for downstream error propagation. As illustrated in Figure~\ref{fig:intuition}, repairing the highest-error state can leave the underlying propagation path unchanged, while repairing a fixed local neighborhood can still leave critical dependencies active. \method{} formalizes error propagation through a coupled node--edge model and searches for a connected core whose repair most reduces residual error amplification. Specifically, it estimates local amplification and node--edge coupling and grows the repair core according to its marginal reduction of the residual spectral radius. The selected core is then provided to the LLM for diagnosis and repair, focusing correction on the states and dependencies responsible for downstream error propagation. We further establish a theoretical bound on long-horizon planning regret, connecting core repair to subsequent planning performance.

\begin{figure}[!ht]

\centering
\includegraphics[width=\linewidth]{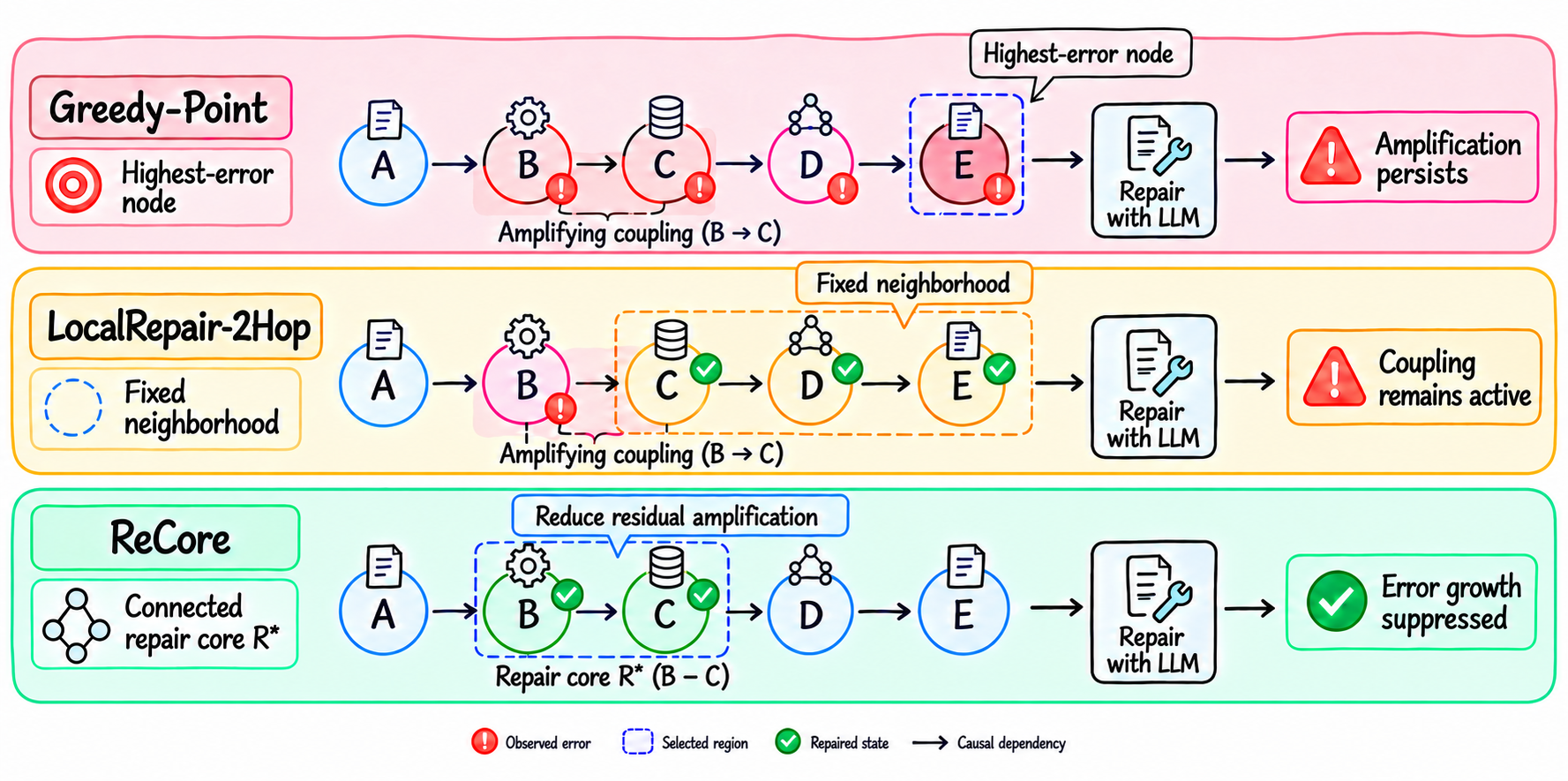}

\caption{\textbf{Repair the core, not the symptom.} Greedy-Point repairs the largest visible error, while LocalRepair-2Hop repairs a fixed neighborhood; both can leave the underlying error-propagation structure active. \method{} instead identifies and repairs the connected core whose correction most reduces residual error amplification.}
\label{fig:intuition}
\Description{Comparison of Greedy-Point, LocalRepair-2Hop, and ReCore. Greedy-Point repairs the largest visible error, LocalRepair-2Hop repairs a fixed neighborhood, and ReCore identifies the connected core whose repair most reduces residual error amplification.}

\end{figure}

We evaluate \method{} through controlled stabilization and root-cause repair experiments on synthetic Agent Calling-Tree graphs and the Who\&When attribution benchmark. On Agent Calling-Tree graphs, where intermediate error signals are available, \method{} identifies compact error-propagating cores and matches the recovery of full-graph repair while requiring up to $7.4\times$ fewer tokens per successful repair. Who\&When provides another agent correction setting, where \method{} also effectively identifies the critical agents and steps responsible for task failures, demonstrating its effectiveness beyond synthetic planning graphs.

Our main contributions are:
\begin{enumerate}[leftmargin=*,noitemsep]
\item \textbf{Planning-graph repair.} We formulate failed long-horizon agent correction as identifying and repairing the compact core responsible for continued error propagation, rather than correcting only visible symptoms or replaying the complete planning graph.
\item \textbf{Engineering repair study.} We systematically evaluate practical repair strategies based on pointwise and TopK errors, temporal windows, graph neighborhoods, cascade tracing, and full-graph LLM correction, exposing the trade-off between targeted but incomplete repair and expensive global correction.
\item \textbf{Repair the core.} We introduce \method{}, which models coupled node--edge error propagation and identifies a connected repair core according to how strongly its correction reduces residual amplification before invoking an LLM.
\item \textbf{Theory and experiments.} We establish a long-horizon planning-regret bound and evaluate \method{} on Agent Calling-Tree graphs and Who\&When, demonstrating compact core repair while identifying the method's boundary when intermediate error signals are unavailable.
\end{enumerate}

\section{Related Work}
\label{sec:related}

\paragraph{Long-horizon planning and world models.}
Long-horizon agents make interdependent decisions over evolving states, making prediction errors increasingly consequential as rollouts progress. Classical planning integrates learning and acting \citep{sutton1991dyna}, while learned world models support multi-step prediction and imagined rollouts \citep{ha2018world,chua2018pets,schrittwieser2020muzero,dreamerv3}. Prior work has studied error accumulation and rollout stability \citep{janner2019trust,asadi2018lipschitz}, as well as graph-based representations of dependencies among states and decisions \citep{anokhin2024arigraph,graph-wm,zhang2021worldgraph,feng2025gwm,gnn-dynamics}. We focus on a complementary post-failure problem: identifying which part of a planning graph should be repaired to reduce subsequent error propagation.

\paragraph{Failure localization and repair in LLM agents.}
LLM agents combine reasoning, interaction, and tool use over multi-step trajectories \citep{react,toolformer,gorilla2024,restgpt2023,toolllm2024}, and existing correction methods use feedback, memory, search, verification, or trajectory revision to recover from failures \citep{llm-repair-plan,reflexion2023,lats2024,cobbe2021verifiers,lightman2023verify,alfworld-repair}. Related work also localizes failures to particular agents, steps, or graph regions \citep{failure-attr-multi-agent,agentlens,fault-subgraph}, while graph methods characterize how errors or signals propagate through structured dependencies \citep{spectral-graph,mpnn,graph-error-amp}. \method{} connects these directions by asking not merely where an error is observed, but which connected region drives its downstream amplification. It jointly models node and edge errors and selects a budgeted repair core according to the reduction in residual propagation, leaving the downstream correction mechanism unchanged.


\begin{table*}[!ht]
\centering
\footnotesize
\setlength{\tabcolsep}{5pt}
\renewcommand{\arraystretch}{1.12}
\caption{Comparison with closely related methods. $\checkmark$ indicates that the corresponding capability is explicitly modeled or optimized.}

\label{tab:related_work_novelty}
\resizebox{0.8\textwidth}{!}{%
\begin{tabular}{lccccccc}
\toprule
\textbf{Method}
& \textbf{Failure}
& \textbf{Dependency}
& \textbf{Node--Edge}
& \textbf{Connected}
& \textbf{Spectral}
& \textbf{Budgeted}
& \textbf{Long-Horizon} \\
& \textbf{Localization}
& \textbf{Structure}
& \textbf{Coupling}
& \textbf{Region}
& \textbf{Objective}
& \textbf{Repair}
& \textbf{Stability} \\
\midrule

Who\&When (2025) \citep{zhang2025whowhen}
& \checkmark & \checkmark & $\times$ & $\times$ & $\times$ & $\times$ & $\times$ \\

CHIEF (2026) \citep{wang2026chief}
& \checkmark & \checkmark & $\times$ & \checkmark & $\times$ & $\times$ & $\times$ \\

FALAT (2026) \citep{rafi2026falat}
& \checkmark & \checkmark & $\times$ & \checkmark & $\times$ & \checkmark & $\times$ \\

EDGE (2026) \citep{hou2026edge}
& \checkmark & \checkmark & $\times$ & $\times$ & $\times$ & $\times$ & $\times$ \\

AgentDebug (2025) \citep{agenterrorbench}
& \checkmark & $\times$ & $\times$ & $\times$ & $\times$ & $\times$ & $\times$ \\

DoVer (2026) \citep{ma2026dover}
& \checkmark & \checkmark & $\times$ & $\times$ & $\times$ & \checkmark & $\times$ \\

CausalFlow (2026) \citep{bonagiri2026causalflow}
& \checkmark & \checkmark & $\times$ & $\times$ & $\times$ & \checkmark & \checkmark \\

K-shield (2026) \citep{avena2026kshield}
& $\times$ & \checkmark & $\times$ & $\times$ & \checkmark & \checkmark & $\times$ \\

\midrule

\rowcolor{gray!15}
\textbf{\method{} (ours)}
& \checkmark
& \checkmark
& \checkmark
& \checkmark
& \checkmark
& \checkmark
& \checkmark \\

\bottomrule
\end{tabular}}

\end{table*}

\paragraph{Relationship to Prior Work.}
Table~\ref{tab:related_work_novelty} summarizes the relationship between \method{} and recent work on failure attribution, targeted debugging, and spectral intervention. Rather than identifying individual errors or searching for minimally sufficient corrections, \method{} focuses on selecting a compact, connected repair region that reduces residual error amplification under a fixed budget. Its main distinction is the combination of a coupled node--edge propagation model, marginal spectral relief for core selection, and a connection between residual amplification and long-horizon planning regret.


\section{Problem Formulation}

\subsection{Planning-Graph Error Propagation}
\label{sec:background}

We characterize how errors propagate through a planning graph and why repairing the largest visible error may not prevent subsequent failures. Because errors can spread through both node states and their dependencies, effective repair should reduce node--edge error propagation rather than target error magnitude alone.

\begin{definition}[Node-state planning error]
Let $X_k$ and $\hat X_k$ denote the reference and agent-maintained node states at planning step $k$, respectively. We define the node-state error as
\begin{equation}
e^X_k=\|\hat X_k-X_k\|_F.
\label{eq:node_state_error}
\end{equation}
\end{definition}
Let $\epsilon_X$ denote the new node-state error introduced during one planning step.

\begin{assumption}[Fixed-edge Lipschitz propagation]
\label{assump:fixed_lip}
Consider a planning graph with fixed adjacency matrix $A$. We assume that its state update is Lipschitz in the node states with constant
\begin{equation}
L_X=L_\sigma\rho(A)\prod_\ell\|W_\ell\|_2,
\label{eq:fixed_lx}
\end{equation}
\end{assumption}
where $L_\sigma$ is the Lipschitz constant of the update nonlinearity and $W_\ell$ denotes the transformation at layer $\ell$.

\begin{theorem}[Fixed-edge error propagation]
\label{thm:fixed_edge}
Under Assumption~\ref{assump:fixed_lip}, suppose $e^X_{k+1}\le L_Xe^X_k+\epsilon_X$, then, for every $k\ge1$,
\begin{equation}
e^X_k\le L_X^ke^X_0+\epsilon_X\frac{L_X^k-1}{L_X-1}.
\label{eq:fixed_growth}
\end{equation}
\end{theorem}
When $L_X=1$, the second term becomes $k\epsilon_X$. Theorem~\ref{thm:fixed_edge} shows that future error depends not only on its current magnitude but also on the structure that propagates it. Thus, a smaller error in a highly amplifying region can have a larger downstream effect than the largest observed error, motivating \method{} to consider both error magnitude and propagation structure.

The fixed-edge setting does not capture errors in planning dependencies themselves, such as using an incorrect observation or dependency for a later decision. We therefore model node and edge errors jointly.

\begin{assumption}[Coupled node--edge propagation]
\label{assump:coupled}
Let $z_k=(e^X_k,e^A_k)^\top$ collect node-state and edge-state errors at planning step $k$, and let $\varepsilon=(\epsilon_X,\epsilon_A)^\top$ denote newly introduced errors. We assume that their one-step propagation admits the upper envelope $z_{k+1}\preceq\Bop z_k+\varepsilon$, with propagation operator
\begin{equation}
\Bop=\begin{pmatrix}L_X & L_A\\ M_X & M_A\end{pmatrix},
\label{eq:coupled_operator}
\end{equation}
\end{assumption}
where all entries of $\Bop$ are nonnegative and $\preceq$ denotes elementwise inequality. Here, $L_X$ and $M_A$ describe within-channel propagation of node and edge errors, while $L_A$ and $M_X$ capture cross-channel coupling between erroneous dependencies and node states.

\begin{theorem}[Coupled node--edge amplification]
\label{thm:coupled_operator}
Under Assumption~\ref{assump:coupled}, the spectral radius of the coupled propagation operator is
\begin{equation}
\Delta_B=\sqrt{(L_X-M_A)^2+4L_AM_X},\qquad \rhoB=\frac{L_X+M_A+\Delta_B}{2}.
\label{eq:coupled_rho}
\end{equation}
If $L_AM_X>0$, then $\rhoB>\max(L_X,M_A)$.
\end{theorem}

Theorem~\ref{thm:coupled_operator} shows that node--edge coupling can amplify errors beyond either channel alone. We therefore define the local coupling score $\kappa_v=L_A(v)M_X(v)$ to measure how strongly node $v$ participates in this coupled propagation. A node with moderate observed error can thus have a large downstream effect when it lies in a strongly coupled region.

\begin{assumption}[Value sensitivity]
\label{assump:value_sensitivity}
For any policy $\pi$, the reward is $L_R$-Lipschitz in the planning state, the policy selection map has sensitivity $\kappa$, and the per-step reward error is bounded by $\epsilon_R$. The coupled planning error satisfies $\|z_k\|\le\epsilon\rhoB^k$, where $z_k=(e^X_k,e^A_k)^\top$.
\end{assumption}

\begin{theorem}[Planning-regret bound]
\label{thm:planning_regret}
Let $\pi^*$ denote the optimal policy under the reference planning states and $\hat\pi$ the policy induced by the erroneous planning graph. Under Assumption~\ref{assump:value_sensitivity},
\begin{equation}
J(\pi^*)-J(\hat\pi)\le2L_R\kappa\epsilon\phiH(\gamma,\rhoB)+2\epsilon_RH,
\label{eq:regret}
\end{equation}
where $\phiH(\gamma,\rho)=\sum_{t=0}^{H-1}(\gamma\rho)^t.$
\end{theorem}

Theorem~\ref{thm:planning_regret} connects error propagation to long-horizon planning performance: a smaller spectral radius reduces the amplification term in the regret bound. For a repaired core $R$, the remaining propagation is governed by the residual spectral radius $\rho(\Bop_{G_f\setminus R})$. Repairing high-error nodes alone can leave this quantity large, whereas \method{} selects a connected core whose correction maximally reduces residual amplification.


\subsection{Planning-Graph Repair}
\label{sec:formulation}

Given a failed planning graph $G_f$ and a limited repair budget, our goal is to identify a compact connected core whose correction minimizes the residual error amplification characterized in Section~\ref{sec:background}.

\begin{definition}[Failure graph]
A failed rollout is represented by
\begin{equation}
G_f=(V,E,\mathbf{X},t^\star),
\label{eq:failure_graph}
\end{equation}
where $V$ is the set of agent decisions, $E\subseteq V\times V$ is the set of typed dependencies, $\mathbf{X}=\{X_v\}_{v\in V}$ contains the node states, and $t^\star\in V$ is the final-answer node. Each node $v\in V$ is associated with
\begin{equation}
\mathbf{s}_v=\bigl(e(v),u(v),c(v)\bigr),
\label{eq:node_attributes}
\end{equation}
\end{definition}
where $e(v)$, $u(v)$, and $c(v)$ denote its observed error, uncertainty, and repair cost, respectively. The ground-truth corrupted region $\mathcal{R}^*\subseteq V$ is used only for evaluation.

\begin{definition}[Core repair objective]
\label{def:obj}
Given budget $K_{\max}$ and the coupled propagation operator $\Bop$ from Equation~\eqref{eq:coupled_operator}, the optimal repair core is
\begin{equation}
R^*=\arg\min_{R\subseteq V:\,G_f[R]\text{ connected},\,|R|\le K_{\max}}\rho(\Bop_{G_f\setminus R}).
\label{eq:repair_objective}
\end{equation}
\end{definition}

Equation~\eqref{eq:repair_objective} directly minimizes the residual amplification governing the error dynamics in Equation~\eqref{eq:coupled_operator}. In contrast, selecting nodes only by $e(v)$ can repair large visible errors while leaving $\rho(\Bop_{G_f\setminus R})$ high. \method{} approximates Equation~\eqref{eq:repair_objective} through greedy connected core selection.


\section{\method{}: Repair the Core}
\label{sec:method}

\begin{figure*}[!ht]
\centering
\includegraphics[width=.75\textwidth]{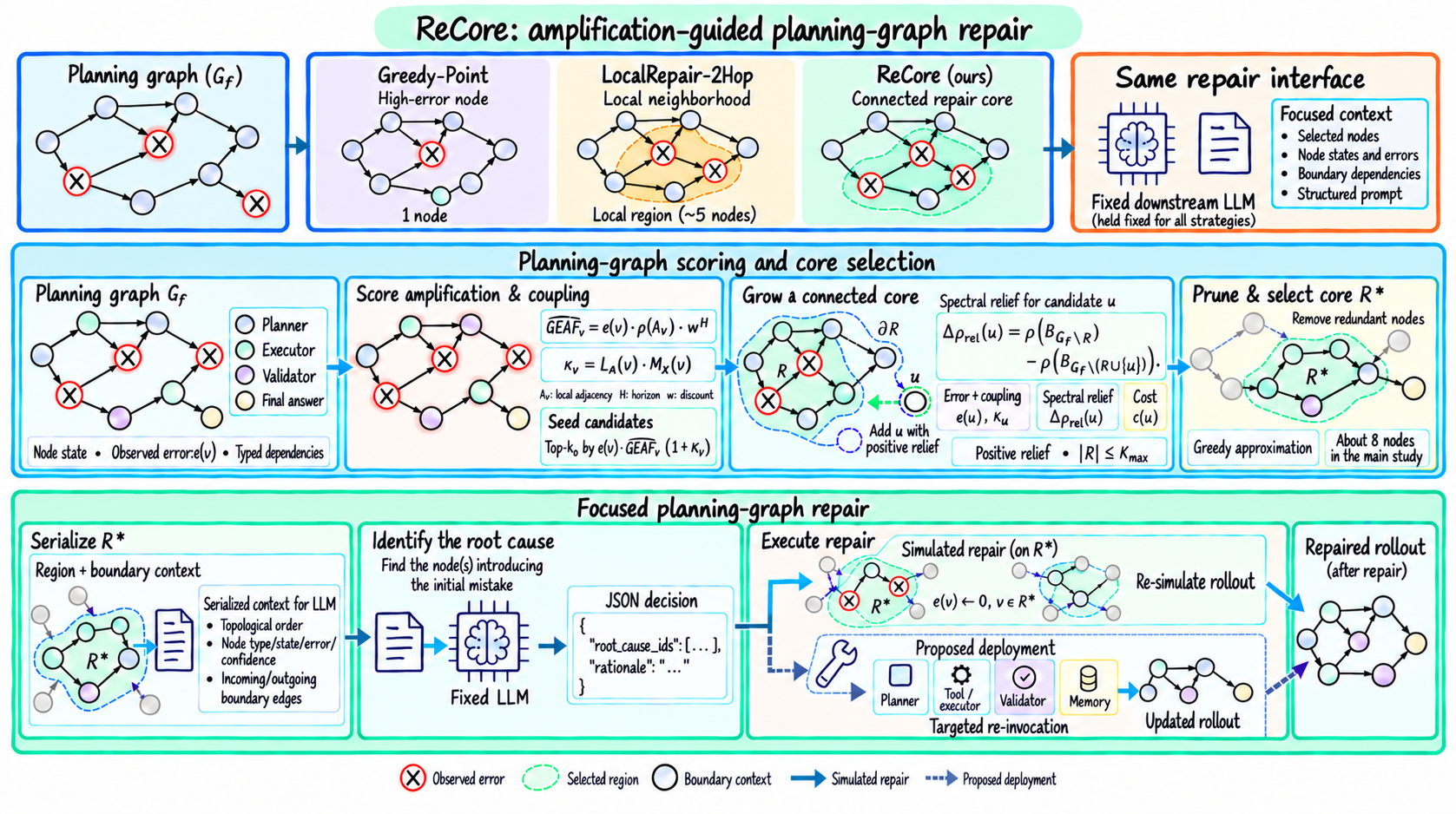}

\caption{\method{} workflow. The method scores nodes using local error amplification and node--edge coupling, initializes candidate cores, expands each core according to its reduction in residual spectral radius, and removes redundant nodes. The selected core is then provided to the LLM for diagnosis and repair.}
\label{fig:pipeline}
\Description{ReCore scores local error amplification and node--edge coupling, initializes candidate cores, expands them by reduction in residual spectral radius, removes redundant nodes, and provides the selected core to an LLM for repair.}

\end{figure*}

\method{} consists of two phases: identifying a compact repair core and correcting it with an LLM. Existing engineering strategies identify repair targets using observed error, temporal proximity, graph distance, or the complete planning graph. \method{} instead follows the objective in Equation~\eqref{eq:repair_objective}, selecting a connected core whose correction reduces the residual spectral radius $\rho(\Bop_{G_f\setminus R})$. Figure~\ref{fig:pipeline} summarizes the workflow.

\subsection{Phase 1: Finding the Repair Core}
\label{sec:region}

\paragraph{Step 1 -- Local amplification.}
Observed error alone does not measure downstream influence. We therefore combine the error at node $v$ with the propagation structure of its local neighborhood:
\begin{equation}
a_v=e(v)\rho(A_v)w^H,
\label{eq:local_amp}
\end{equation}
where $A_v$ is the adjacency matrix of the $H$-hop subgraph around $v$, $\rho(A_v)$ is its spectral radius, and $w\in(0,1)$ discounts propagation over the horizon. Thus, $e(v)$ measures the current error, while $\rho(A_v)$ measures the local structure through which that error can propagate. Hyperparameters are given in Appendix~\ref{app:hyper}.

\paragraph{Step 2 -- Node--edge coupling.}
Following Theorem~\ref{thm:coupled_operator}, we define
\begin{equation}
\kappa_v=L_A(v)M_X(v),
\label{eq:local_coupling}
\end{equation}
where
\begin{equation}
L_A(v)=w\beta_A\bar d_{\rm in}(v)\bar e(v),\qquad M_X(v)=w\beta_X\bar d_{\rm out}(v)\bar e(v).
\label{eq:local_channels}
\end{equation}
Here, $\bar d_{\rm in}(v)$ and $\bar d_{\rm out}(v)$ measure the average incoming and outgoing dependency diversity around $v$, $\bar e(v)$ is the mean local error, and $\beta_A,\beta_X$ are calibration constants. A larger $\kappa_v$ indicates stronger interaction between node and edge errors.

\paragraph{Step 3 -- Seed selection.}
We combine observed error, local amplification, and node--edge coupling to define
\begin{equation}
\mathrm{score}_{\rm seed}(v)=e(v)a_v(1+\kappa_v).
\label{eq:seed_score}
\end{equation}
The top-$k_0$ nodes are used as candidate seeds, and each seed initializes a connected candidate core.

\paragraph{Step 4 -- Core expansion.}
Starting from $R=\{s\}$, \method{} repeatedly adds a neighboring node that provides the largest gain:
\begin{equation}
u^*=\arg\max_{u\in\partial R}g_R(u),
\label{eq:grow}
\end{equation}
where
\begin{equation}
g_R(u)=\lambda_1e(u)(1+\kappa_u)+\lambda_2\Delta\rho_R(u)-\lambda_3c_R(u).
\label{eq:grow_score}
\end{equation}
The reduction in residual spectral radius is
\begin{equation}
\Delta\rho_R(u)=\rho(\Bop_{G_f\setminus R})-\rho(\Bop_{G_f\setminus(R\cup\{u\})}),
\label{eq:delta_rho}
\end{equation}
and the redundancy cost is
\begin{equation}
c_R(u)=1+\frac{|N(u)\cap R|}{|R|}.
\label{eq:cost}
\end{equation}
Here, $\partial R$ contains nodes adjacent to the current core. The three terms favor nodes with observed error and coupling, nodes that reduce residual amplification, and nodes that add nonredundant information, respectively. Expansion stops when the budget $K_{\max}$ is reached or no candidate provides a positive reduction.

Unlike TopK or fixed-neighborhood repair, this procedure does not select nodes solely because they have large errors or are close to an observed failure. It directly uses the residual spectral radius from Equation~\eqref{eq:repair_objective} to determine which neighboring nodes should be included.

\paragraph{Step 5 -- Pruning and selection.}
After expansion, we remove any node whose inclusion does not reduce the residual spectral radius. Specifically, $v\in R$ is removed if
\begin{equation}
\rho(\Bop_{G_f\setminus(R\setminus\{v\})})\le\rho(\Bop_{G_f\setminus R}).
\label{eq:pruning}
\end{equation}
Among the resulting candidate cores, we select
\begin{equation}
R^*=\arg\max_R\mathrm{Score}(R),
\label{eq:final_core}
\end{equation}
with
\begin{equation}
\mathrm{Score}(R)=\frac{\left(\sum_{v\in R}e(v)(1+\kappa_v)\right)\Delta\rho(R)}{1+|R|},
\label{eq:region_score}
\end{equation}
where
\begin{equation}
\Delta\rho(R)=\rho(\Bop_{G_f})-\rho(\Bop_{G_f\setminus R}).
\label{eq:region_delta_rho}
\end{equation}
This favors compact connected cores that cover relevant errors while substantially reducing residual error propagation.

Algorithm~\ref{alg:recore} summarizes the complete selection procedure.

\begin{algorithm}[!ht]
\caption{\method{} Core Selection}
\label{alg:recore}
\begin{algorithmic}[1]
\Require Failure graph $G_f$, budget $K_{\max}$, parameters $(H,w,k_0,\lambda_1,\lambda_2,\lambda_3)$
\Ensure Connected repair core $R^*$
\State Compute $a_v\gets e(v)\rho(A_v)w^H$ and $\kappa_v\gets L_A(v)M_X(v)$ for each $v\in V$
\State $\mathcal{S}\gets\text{top-}k_0$ nodes by $e(v)a_v(1+\kappa_v)$
\State $\mathcal{C}\gets\emptyset$
\For{each seed $s\in\mathcal{S}$}
    \State $R\gets\{s\}$
    \While{$|R|<K_{\max}$}
        \State $u^*\gets\arg\max_{u\in\partial R}g_R(u)$
        \If{$\Delta\rho_R(u^*)>0$}
            \State $R\gets R\cup\{u^*\}$
        \Else
            \State \textbf{break}
        \EndIf
    \EndWhile
    \State Remove $v\in R$ satisfying Equation~\eqref{eq:pruning}
    \State $\mathcal{C}\gets\mathcal{C}\cup\{R\}$
\EndFor
\State \Return $R^*\gets\arg\max_{R\in\mathcal{C}}\mathrm{Score}(R)$
\end{algorithmic}
\end{algorithm}

\subsection{Phase 2: LLM-Based Repair}
\label{sec:repair}

Once \method{} selects the repair core $R^*$, only the selected nodes, their states and observed errors, and the dependencies within and around the core are provided to the LLM. The LLM identifies the responsible node(s) and produces a structured repair decision. In simulation, we correct the selected node errors and continue the planning rollout; in a deployed agent, the repair can trigger targeted re-execution of the corresponding planning step, tool call, or validation step. Prompt details are provided in Appendix~\ref{app:prompt}.

This design separates core identification from LLM-based correction. Rather than asking the LLM to locate the failure from the complete planning graph, \method{} first identifies the compact core responsible for downstream error propagation. Figure~\ref{fig:tradeoff} reports the core-size--quality tradeoff, and Table~\ref{tab:llm} reports the corresponding token--quality tradeoff.

\section{Dataset: Agent Calling-Tree Testbed}
\label{sec:data}

We construct a synthetic heterogeneous planning-graph testbed that captures the typed dependencies among decisions in long-horizon agents.

\paragraph{Planning-graph structure.}
Each instance contains a typed planning graph with nodes representing planner, executor, validator, checker, aggregator, reporter, logger, error-handler, and final-answer decisions. Edges represent dependencies such as calls, validation, reporting, error routing, triggering, and logging. Together, these nodes and edges form the failure graph $G_f$ defined in Equation~\eqref{eq:failure_graph}. Full dataset statistics are provided in Appendix~\ref{app:data}.

\paragraph{Failure injection.}
We inject errors into upstream planning nodes and propagate them through their downstream dependencies. As a result, the largest observed error can occur downstream from the node that initiated the failure, creating a distinction between visible symptoms and the underlying repair core. The ground-truth corrupted region $\mathcal{R}^*$ is used only for evaluation.

\paragraph{Multi-step error metrics.}
After correcting a selected core $R$, we propagate the remaining node and edge errors according to Equation~\eqref{eq:coupled_operator}. We report long-horizon NodeMSE, GrowthSlope, and the reduction in residual spectral radius $\rho(\Bop_{G_f\setminus R})$. These metrics measure both immediate error reduction and the ability of a repair to suppress subsequent error propagation.

\section{Experiments}
\label{sec:experiments}

\subsection{Baselines: Engineering Repair Strategies}

We evaluate \method{} against a broad set of engineering repair strategies for failed planning graphs. \textbf{Greedy-Point (K=1)} repairs the node with the largest observed error, representing direct symptom correction. \textbf{TopK-Point (K=3/5)} selects the top-$K$ highest-error nodes without requiring connectivity. \textbf{Window-2/4/8-Point} selects the temporal window of $k$ consecutive steps with the largest mean error, reflecting temporal failure localization in multi-step agent trajectories \citep{zhang2025whowhen}. \textbf{LocalRepair-2/3Hop} selects the $k$-hop neighborhood around the highest-error node to include nearby dependent decisions. \textbf{CascadeRepair} follows downstream dependencies until the observed error falls below a threshold. \textbf{Oracle} repairs the ground-truth corrupted region $\mathcal{R}^*$ and serves as a reference for perfect corrupted-region identification. \textbf{LLM-based repair} provides regions selected by these strategies, including the complete planning graph, to the same LLM for root-cause identification. Finally, \textbf{\method{} (ours)} selects a compact connected repair core using local amplification, node--edge coupling, and reduction in residual spectral radius according to Algorithm~\ref{alg:recore}.

\subsection{Main Results}
\label{sec:results}

Table~\ref{tab:main} reports the main results. \method{} performs consistently well across the evaluation metrics, achieving the strongest residual spectral-radius reduction among non-oracle methods while maintaining low prediction error through $H=32$ with a compact repair region. Although CascadeRepair attains slightly better IoU and MSE at $H=4$, its error grows substantially at longer horizons, whereas \method{} remains stable as the rollout horizon increases. These results suggest that targeting error-amplifying dependencies is particularly effective for sustained long-horizon repair.

\begin{table*}[!ht]
\centering
\small
\setlength{\tabcolsep}{3pt}
\caption{Main comparison on the Agent Calling-Tree testbed. Size is the mean repair-region size; Conn is the percentage of connected repair regions; $\rhoB$-red $=\rho(\Bop_{G_f})-\rho(\Bop_{G_f\setminus R})$. Best values among non-oracle methods in each column are shown in \textbf{bold}. Oracle repairs the ground-truth injected corrupted region and serves as an attribution reference, rather than optimizing residual spectral radius.}

\label{tab:main}
\resizebox{0.65\textwidth}{!}{%
\begin{tabular}{lrrrrrrrr}
\toprule
\textbf{Method}
  & \textbf{Size}$\downarrow$
  & \textbf{Conn}\%$\uparrow$
  & \textbf{IoU}$\uparrow$
  & \bm{$\rhoB$}\textbf{-red}$\uparrow$
  & \textbf{MSE@4}$\downarrow$
  & \textbf{MSE@16}$\downarrow$
  & \textbf{MSE@32}$\downarrow$ \\
\midrule
\multicolumn{8}{l}{\textit{Pointwise methods}} \\
Greedy-Point(K=1)   & \textbf{1.0} & \textbf{100} & 0.144 & 0.22 {\scriptsize$\pm$0.06} & 99.0 & 184.7 & 185.2 {\scriptsize$\pm$18} \\
TopK-Point(K=3)     & 3.0 & 36 & 0.433 & 0.68 {\scriptsize$\pm$0.09} & 69.3 & 143.3 & 144.1 {\scriptsize$\pm$15} \\
TopK-Point(K=5)     & 5.0 & 48 & 0.630 & 1.08 {\scriptsize$\pm$0.10} & 46.6 & 107.3 & 107.6 {\scriptsize$\pm$14} \\
\midrule
\multicolumn{8}{l}{\textit{Window-based methods}} \\
Window-2-Point      & 2.0 & 30 & 0.251 & 0.50 {\scriptsize$\pm$0.08} & 91.2 & 184.8 & 185.5 {\scriptsize$\pm$18} \\
Window-4-Point      & 4.0 & 12 & 0.317 & 0.82 {\scriptsize$\pm$0.09} & 80.0 & 182.9 & 183.8 {\scriptsize$\pm$17} \\
Window-8-Point      & 8.0 & 22 & 0.360 & 1.30 {\scriptsize$\pm$0.09} & 61.2 & 169.3 & 170.8 {\scriptsize$\pm$16} \\
\midrule
\multicolumn{8}{l}{\textit{Neighborhood / cascade methods}} \\
LocalRepair-2Hop    & 18.3 & \textbf{100} & 0.362 & 1.81 {\scriptsize$\pm$0.06} & 27.5 & 91.8 & 92.8 {\scriptsize$\pm$16} \\
CascadeRepair       & 7.3 & 96 & \textbf{0.871} & 1.26 {\scriptsize$\pm$0.10} & \textbf{4.9} & 57.6 & 61.1 {\scriptsize$\pm$10} \\
\midrule
\method{} (ours)    & 8.3 & 94 & 0.845 & \textbf{1.95 {\scriptsize$\pm$0.05}} & 6.0 & \textbf{6.3} & \textbf{6.3 {\scriptsize$\pm$1.1}} \\
Oracle              & 8.1 & 100 & 1.00 & 1.32 {\scriptsize$\pm$0.10} & 0.0 & 0.0 & 0.0 \\
\bottomrule
\end{tabular}%
}

\end{table*}

\begin{figure}[!ht]

\centering
\includegraphics[width=\linewidth]{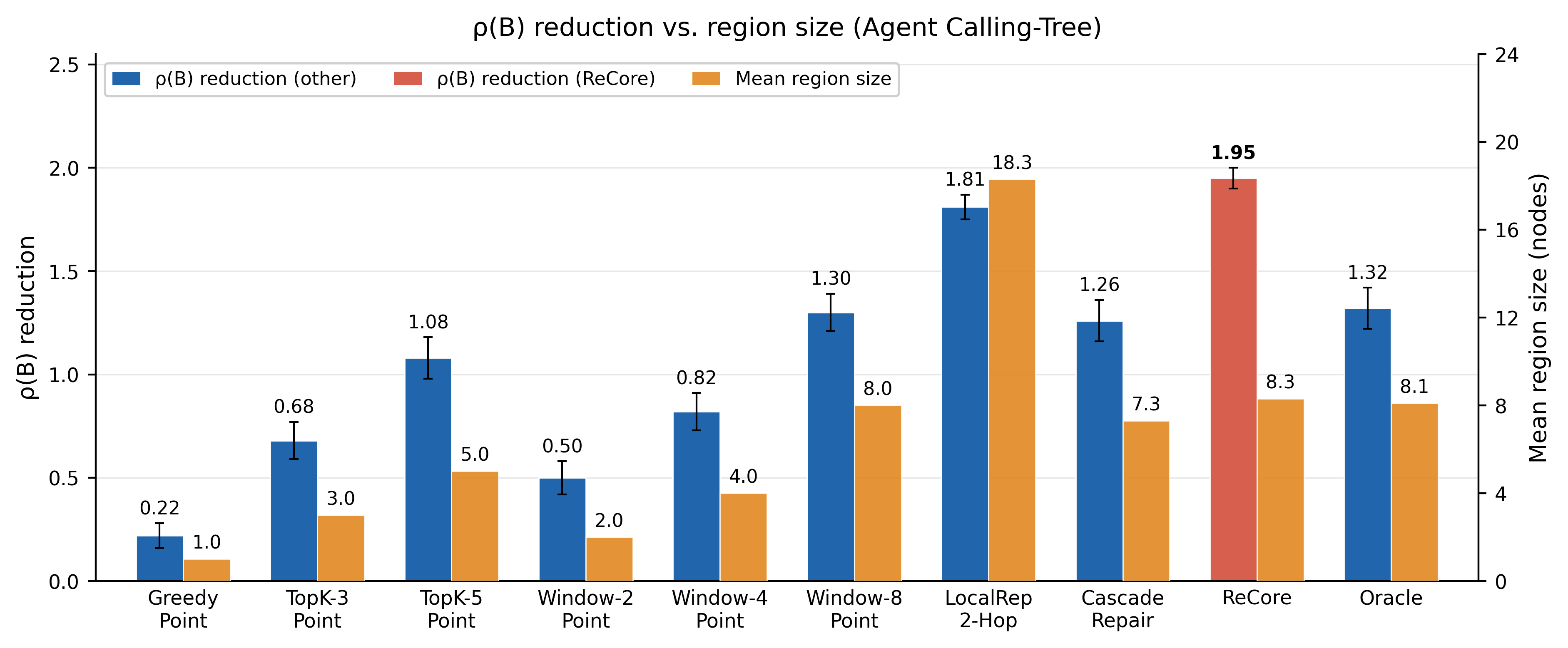}

\caption{Reduction in residual spectral radius across repair strategies.}
\label{fig:rho_reduction}
\Description{Comparison of residual spectral-radius reduction across repair strategies. ReCore achieves a large reduction with a compact connected repair core.}

\end{figure}

Greedy, TopK, and window-based methods prioritize nodes with large observed errors, but their long-horizon MSE remains high. As shown in Figure~\ref{fig:rho_reduction}, these methods also produce substantially smaller reductions in residual spectral radius. This supports the distinction between where an error is most visible and which part of the planning graph should be repaired to prevent further propagation.

\begin{figure}[!ht]

\centering
\includegraphics[width=\linewidth]{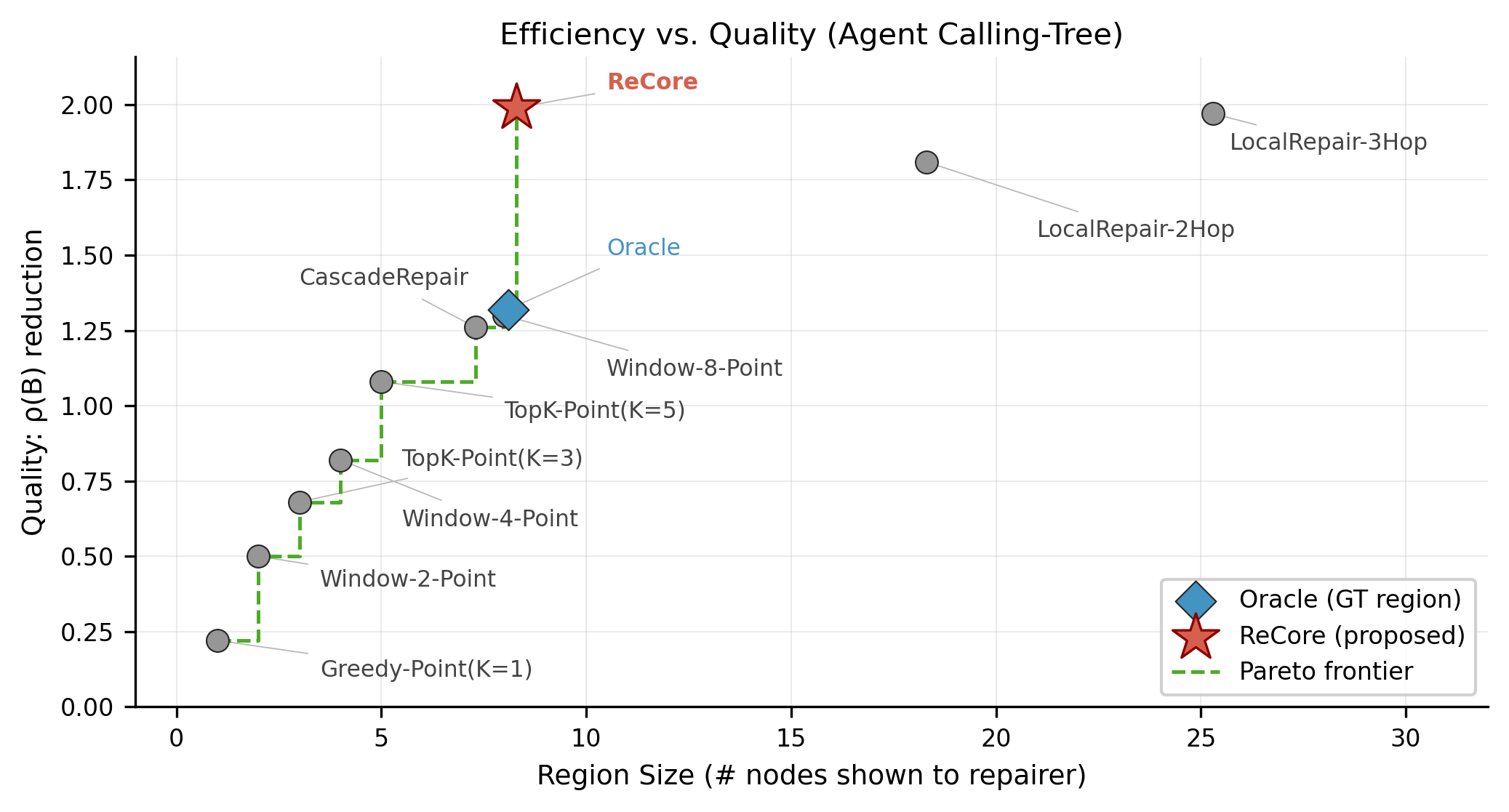}

\caption{Repair-size--quality tradeoff for planning-graph repair. The $x$-axis shows the mean number of selected nodes and the $y$-axis shows normalized reduction in residual spectral radius.}
\label{fig:tradeoff}
\Description{Scatter plot of mean repair-region size against normalized residual spectral-radius reduction. ReCore combines a compact repair core with a large reduction in residual spectral radius.}

\end{figure}

\method{} reduces the residual spectral radius by $1.99$ while maintaining low error through horizon $32$, consistent with Theorem~\ref{thm:planning_regret}. It achieves slightly stronger spectral reduction than LocalRepair-3Hop ($1.99$ vs.\ $1.97$) using only 8.3 rather than 25.3 nodes; Figure~\ref{fig:tradeoff} and the budget sweep in Figure~\ref{fig:budget} further demonstrate this efficiency. Meanwhile, CascadeRepair achieves a high corruption overlap ($0.871$ IoU) but reaches an MSE of $61.1$ at horizon $32$, showing that overlap alone does not ensure stable repair.

\subsection{LLM-Based Repair and Cross-Model Evaluation}
\label{sec:llm}

We next evaluate whether compact repair-core selection improves LLM-based failure identification. For each instance, the selected region is represented in structured natural language and provided to the same LLM, which predicts the node responsible for the failure. We first compare all selection methods using GPT-4o-mini~\citep{openai2024gpt4omini}, and then evaluate GPT-4o~\citep{openai2024gpt4osystemcard} and Gemini-2.5-Flash~\citep{google2026gemini25flash} to test whether the effect persists across LLMs.

\begin{table}[!ht]

\centering\small\setlength{\tabcolsep}{3pt}
\caption{LLM root-cause identification using GPT-4o-mini ($n=150$ prompts; 50 graphs $\times$ 3 seeds). Rec-E/T/H = Rec-Exact/Type/Hop2; Tok = mean tokens; Sz = mean region size. Best values in each column are shown in \textbf{bold}. \method{} has a cross-seed CV of 10.6\%, close to Full-Graph (8.3\%), while other compact-region methods have CVs of 39--173\%.}

\label{tab:llm}
\begin{tabular}{l@{ }c@{ }c@{ }c@{ }r@{ }r}
\toprule
\textbf{Method}
  & \textbf{Rec-E}$\uparrow$
  & \textbf{Rec-T}$\uparrow$
  & \textbf{Rec-H}$\uparrow$
  & \textbf{Tok}$\downarrow$
  & \textbf{Sz}$\downarrow$ \\
\midrule
\method{} (ours)         & \textbf{0.500} & \textbf{0.500} & \textbf{0.987} & 855  & 8.3 \\
Full-Graph-LLM           & 0.240 & 0.247 & 0.953 & 2157 & 25.7 \\
LLMRepair-Full-Plan-LLM  & 0.180 & 0.180 & 0.980 & 2299 & 25.7 \\
LocalRepair-2Hop-LLM     & 0.167 & 0.173 & 0.953 & 1596 & 18.7 \\
TraceScan-Full-LLM       & 0.167 & 0.180 & 0.933 & 1284 & 25.7 \\
TopK-5-LLM               & 0.107 & 0.127 & 0.893 & 609  & 5.0 \\
Window-8-LLM             & 0.087 & 0.107 & 0.887 & 735  & 8.0 \\
Window-4-LLM             & 0.060 & 0.080 & 0.813 & 515  & 4.0 \\
TraceScan-w2-LLM         & 0.020 & 0.080 & 0.760 & 397  & 2.0 \\
TraceScan-w1-LLM         & 0.007 & 0.033 & 0.700 & \textbf{354} & \textbf{1.0} \\
Greedy-Point-LLM         & 0.007 & 0.033 & 0.700 & 363  & \textbf{1.0} \\
\bottomrule
\end{tabular}

\end{table}

\paragraph{Compact repair cores improve failure identification.}
Table~\ref{tab:llm} shows that \method{} achieves the highest exact root-cause recall (0.500) while using 855 tokens, less than half the context of Full-Graph-LLM. This suggests that selecting propagation-relevant dependencies provides a more informative context than either sparse pointwise regions or the full graph.

\begin{figure}[!ht]
\centering
\includegraphics[width=0.8\linewidth]{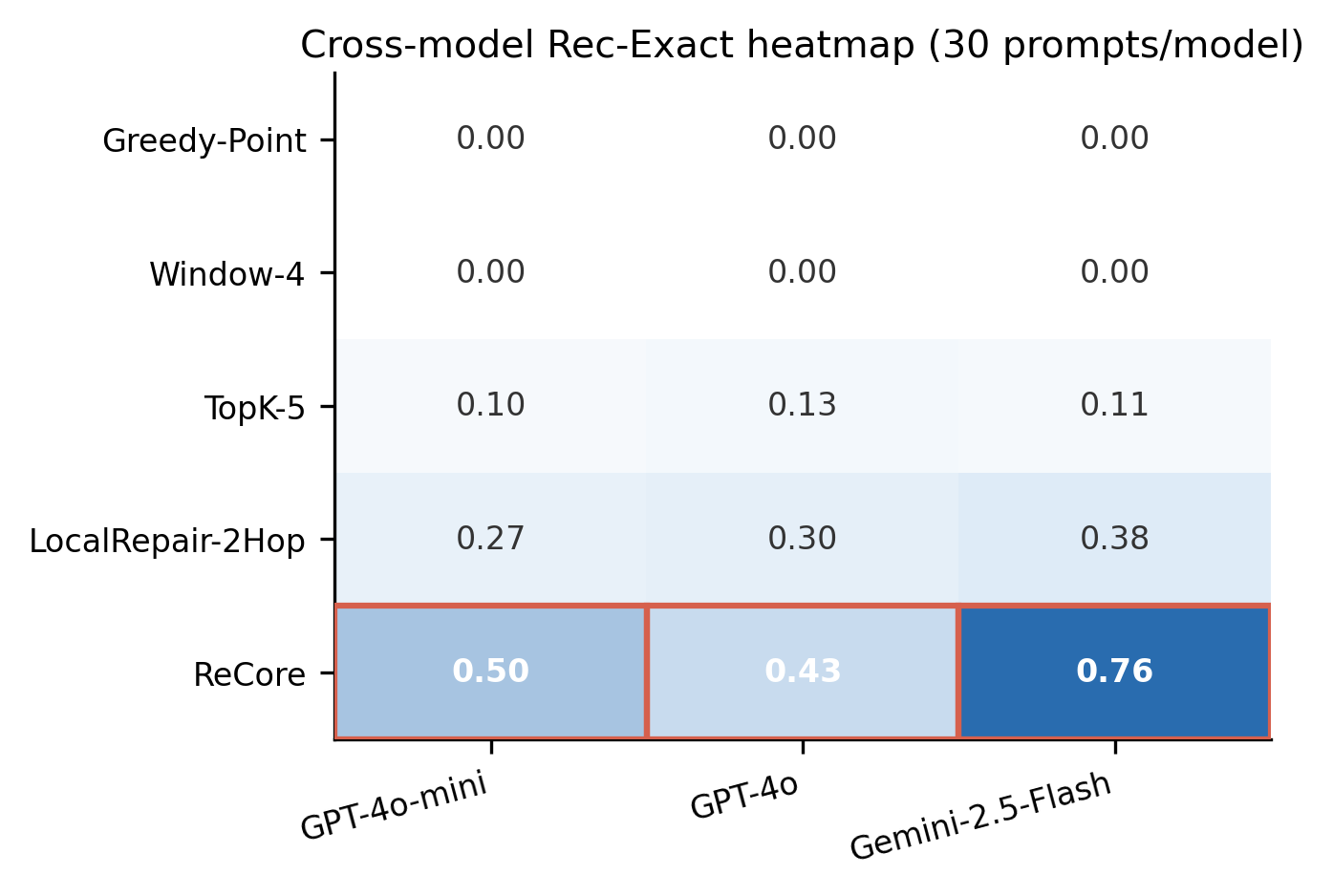}

\caption{Cross-model root-cause identification. Each model receives the same structured prompt format over the region selected by each repair strategy. Cells report exact root-cause recall over valid structured responses from 30 prompts per model.}
\label{fig:llm_multiapi_heatmap}
\Description{Heatmap of exact root-cause recall for repair-region selectors across GPT-4o-mini, GPT-4o, and Gemini-2.5-Flash. ReCore achieves the highest recall across the evaluated models.}

\end{figure}

\paragraph{The advantage persists across LLMs.}
Figure~\ref{fig:llm_multiapi_heatmap} shows the same advantage across GPT-4o-mini, GPT-4o, and Gemini-2.5-Flash. Because the prompt format is fixed within each comparison, the consistent improvement supports repair-core selection, rather than changes to the downstream LLM, as the source of the gain.

\subsection{Horizon Stability}
\label{sec:horizon-stability}

Figure~\ref{fig:mse_profile} shows that \method{} maintains low NodeMSE through $H=32$ with only 8.3 repaired nodes, whereas pointwise and window-based methods accumulate substantial error and LocalRepair-3Hop requires a much larger repair region.

\begin{figure}[!ht]
\centering
\includegraphics[width=0.8\linewidth]{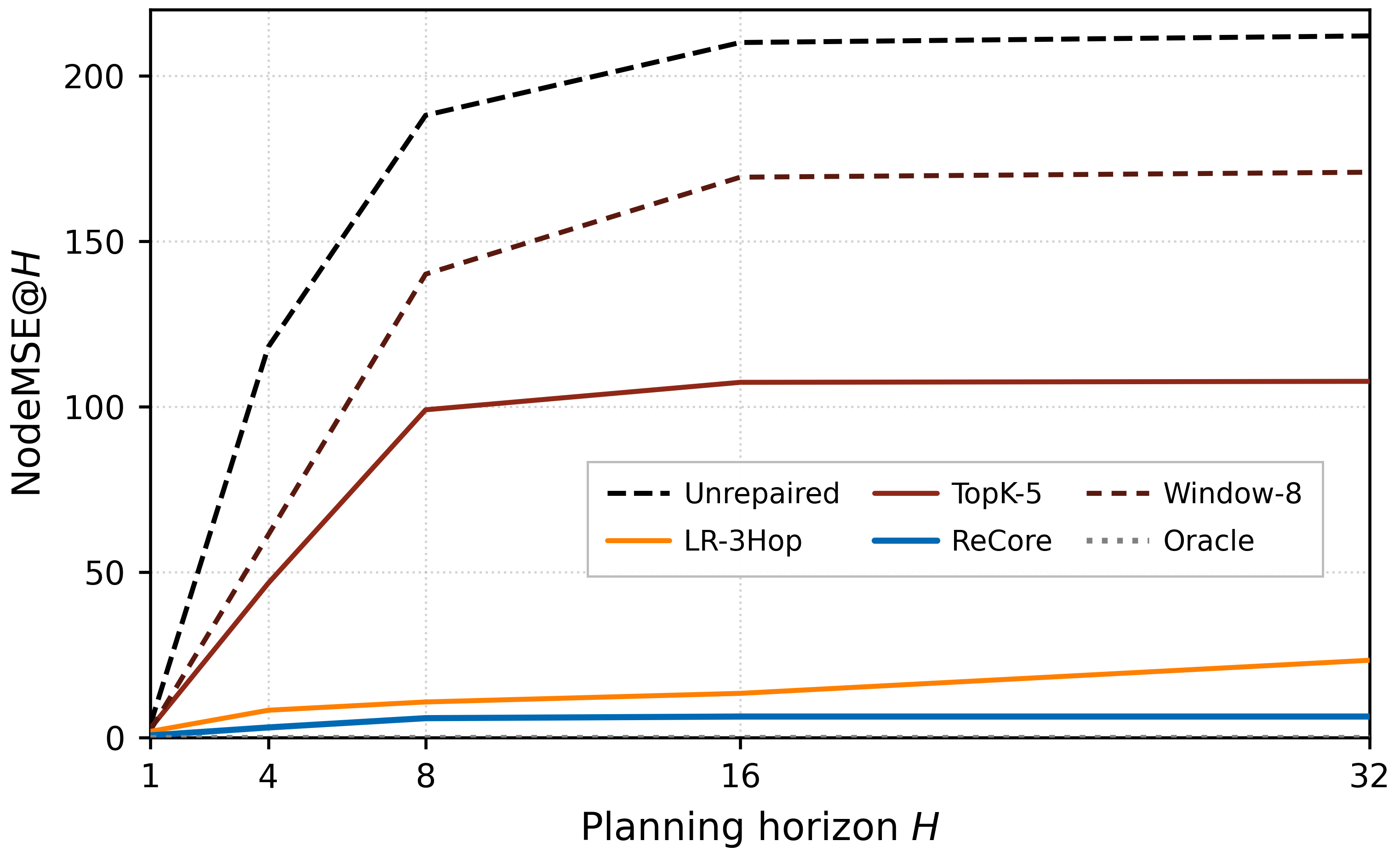}

\caption{NodeMSE across planning horizons ($n=50$). Pointwise and window-based repairs accumulate substantial long-horizon error, while \method{} maintains low error through $H=32$ with a compact repair region.}
\label{fig:mse_profile}
\Description{NodeMSE across planning horizons from 1 to 32 for unrepaired, TopK-5, Window-8, LocalRepair-3Hop, ReCore, and Oracle.}

\end{figure}

\subsection{Ablation Study}
\label{app:ablation}

Table~\ref{tab:ablation} evaluates the contribution of each component to spectral repair and long-horizon stability.

\begin{table}[!ht]
\centering\small
\caption{Ablation on the calling-tree testbed ($n=50$). $\rho$-red denotes residual spectral-radius reduction; Sz denotes mean region size. Best values are shown in \textbf{bold}.}

\label{tab:ablation}
\setlength{\tabcolsep}{4pt}
\resizebox{0.7\linewidth}{!}{%
\begin{tabular}{lrrr}
\toprule
\textbf{Config} & $\bm{\rho}$\textbf{-red}$\uparrow$ & \textbf{MSE@32}$\downarrow$ & \textbf{Sz}$\downarrow$ \\
\midrule
Full \method{}          & \textbf{1.95} & \textbf{6.3} & 8.3 \\
$-$GEAF                 & 1.72 & 11.4 & 8.3 \\
$-$Coupling             & 1.80 & 8.1 & 8.3 \\
$-\rho$-relief          & 1.30 & 13.8 & 7.6 \\
$-$Pruning              & 1.95 & 6.3 & 13.5 \\
$-$Growing              & 0.65 & 210 & \textbf{1.0} \\
\bottomrule
\end{tabular}}

\end{table}

As shown in Table~\ref{tab:ablation}, each component contributes to effective repair, with region growing having the largest impact: removing it reduces $\rho$-red from $1.95$ to $0.65$ and increases MSE@32 from $6.3$ to $210$. GEAF, coupling, and spectral relief further improve repair quality, while pruning mainly keeps the selected region compact without sacrificing performance.

\subsection{Who\&When Attribution Benchmark}
\label{app:attribution}

Table~\ref{tab:attribution} evaluates \method{} on the Who\&When multi-agent failure attribution benchmark~\citep{failure-attr-multi-agent} ($n=30$ algorithm-generated traces), where explicit per-node error magnitudes are unavailable.

\begin{table}[!ht]
\centering\small\setlength{\tabcolsep}{4pt}
\caption{Who\&When attribution benchmark (GPT-4o-mini). Rec-Exact measures whether the ground-truth mistake step is selected; Rec-Agent measures whether a message from the ground-truth mistake agent is selected; Size is the mean region size. Best values are shown in \textbf{bold}.}

\label{tab:attribution}
\begin{tabular}{lrrr}
\toprule
\textbf{Method} & \textbf{Rec-Exact}$\uparrow$ & \textbf{Rec-Agent}$\uparrow$ & \textbf{Size}$\downarrow$ \\
\midrule
LocalRepair-2Hop   & 0.60 & 0.63 & 4.9 \\
TraceScan-w4-Point & 0.37 & 0.70 & 4.0 \\
\method{}          & \textbf{0.80} & \textbf{0.87} & 2.1 \\
LastError-Point    & 0.00 & 0.30 & \textbf{1.0} \\
\bottomrule
\end{tabular}

\end{table}

\method{} achieves the highest Rec-Exact (0.80) and Rec-Agent (0.87) while selecting only 2.1 nodes on average. Compared with LocalRepair-2Hop, it improves exact-step recall by 0.20 while using less than half the repair region, showing that its structural selection remains effective even without explicit per-node error magnitudes.

\subsection{Benchmark Topology Generalization}
\label{app:benchmarks}

We further test whether \method{} generalises beyond the calling-tree testbed using three benchmark-inspired graph families based on SWE-bench~\cite{swebench}, WebArena~\cite{webarena}, and AgentBench-OS~\citep{agentbench}. These topology simulations ($n=50$ per family) capture representative branching, sequential navigation, and command-execution structures, respectively, rather than executing the original benchmark systems.

\begin{table}[!ht]
\centering\small\setlength{\tabcolsep}{3pt}
\caption{Benchmark-topology generalization ($n=50$ per topology). $\rho$-r denotes residual spectral-radius reduction and Sz denotes mean repair-region size. Best $\rho$-r values are shown in \textbf{bold}.}

\label{tab:benchmarks}
\begin{tabular}{lrr@{ }rr@{ }rr}
\toprule
& \multicolumn{2}{c}{\textbf{SWE-bench}} &
  \multicolumn{2}{c}{\textbf{WebArena}} &
  \multicolumn{2}{c}{\textbf{AgentBench}} \\
& \multicolumn{2}{c}{\scriptsize$\rho_0=1.92$} &
  \multicolumn{2}{c}{\scriptsize$\rho_0=1.24$} &
  \multicolumn{2}{c}{\scriptsize$\rho_0=1.39$} \\
\cmidrule(lr){2-3}\cmidrule(lr){4-5}\cmidrule(lr){6-7}
\textbf{Method} & $\rho$-r & Sz & $\rho$-r & Sz & $\rho$-r & Sz \\
\midrule
Greedy-Pt & 1.43 & 1  & 0.59 & 1  & 0.88 & 1 \\
Window-4  & 1.44 & 4  & 0.89 & 4  & 1.08 & 4 \\
TopK-5    & 1.45 & 5  & 0.84 & 5  & 1.06 & 5 \\
LR-2H     & 1.65 & 18 & 0.47 & 8  & 1.26 & 11 \\
LR-3H     & \textbf{1.92} & 21 & 1.14 & 11 & 1.35 & 12 \\
Cascade   & 1.55 & 15 & 1.08 & 13 & \textbf{1.39} & 12 \\
\midrule
\method{} & 1.79 & 17 & \textbf{1.24} & 14 & \textbf{1.39} & 11 \\
\bottomrule
\end{tabular}

\end{table}

Table~\ref{tab:benchmarks} shows that \method{} remains effective across distinct graph structures, achieving the maximum spectral reduction on WebArena ($1.24$) and matching the best result on AgentBench ($1.39$) with a smaller repair region. This demonstrates that its spectral advantage generalises beyond the calling-tree topology.

\subsection{Budget Efficiency Analysis}
\label{app:budget}

Figure~\ref{fig:budget} evaluates spectral reduction under explicit repair-region budgets. \method{} reaches strong spectral relief with small regions and saturates around $K=8$, consistent with the $\Delta\rho_{\rm rel}>0$ growing criterion in Equation~\eqref{eq:grow}.

\begin{figure}[!ht]
\centering
\includegraphics[width=0.8\linewidth]{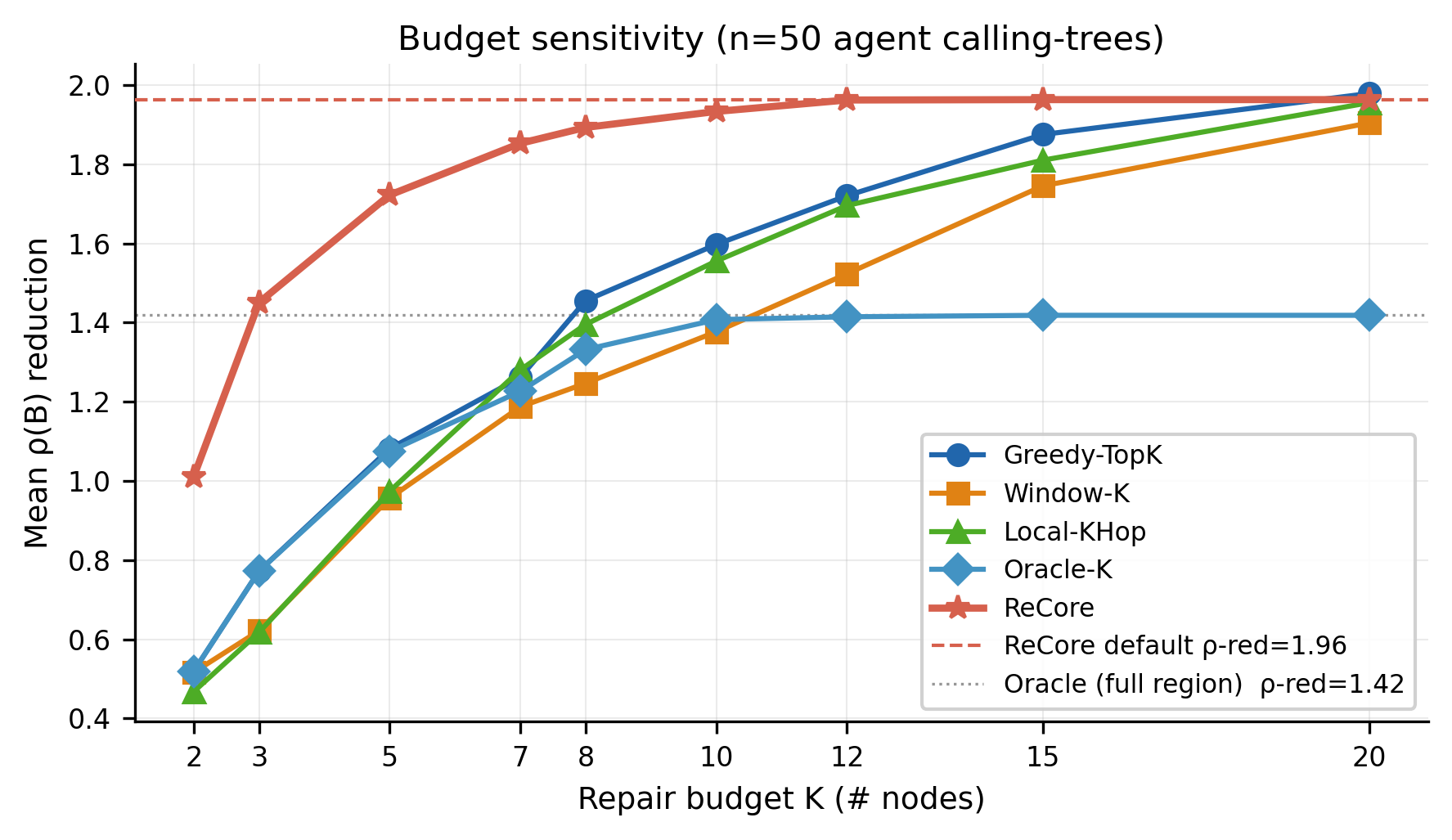}

\caption{Budget sensitivity of residual spectral-radius reduction ($n=50$).}
\label{fig:budget}
\Description{Residual spectral-radius reduction as a function of repair-region budget. ReCore achieves strong reduction with small regions and approaches its unconstrained result as the budget increases.}

\end{figure}

At $K=3$, \method{} achieves a $\rho$-reduction of $1.45$, compared with $0.62$--$0.77$ for engineering baselines. This advantage arises because growing prioritizes nodes that reduce the residual amplifier rather than expanding toward downstream symptoms.

\subsection{Cross-Model Study}
\label{app:multiapi}

Table~\ref{tab:multiapi} extends the comparison to GPT-4o-mini, GPT-4o, and Gemini-2.5-Flash. \method{} achieves the highest exact root-cause recall across all three models.

\begin{table}[!ht]
\centering\small\setlength{\tabcolsep}{4pt}
\caption{Cross-model exact root-cause recall (30 prompts per model). Best values are shown in \textbf{bold}.}

\label{tab:multiapi}
\resizebox{0.9\linewidth}{!}{%
\begin{tabular}{lrrr}
\toprule
\textbf{Method} & \textbf{GPT-4o-mini} & \textbf{GPT-4o} & \textbf{Gemini-2.5-Flash} \\
\midrule
\method{} (ours)  & \textbf{0.500} & \textbf{0.433} & \textbf{0.759} \\
LocalRepair-2Hop  & 0.267 & 0.300 & 0.379 \\
TopK-5            & 0.100 & 0.133 & 0.111 \\
Window-4          & 0.000 & 0.000 & 0.000 \\
Greedy-Point      & 0.000 & 0.000 & 0.000 \\
\bottomrule
\end{tabular}}

\end{table}

Across all three models, \method{} consistently achieves the highest exact root-cause recall, reaching $0.500$, $0.433$, and $0.759$ on GPT-4o-mini, GPT-4o, and Gemini-2.5-Flash, respectively. This cross-model consistency suggests that the gain comes from selecting a more informative repair region rather than from a particular downstream LLM.

\subsection{Cascade Gain Robustness}
\label{app:cascade}

We vary the cascade gain $\alpha\in[0.7,1.4]$, which controls the strength of error propagation along graph edges in the failure process defined in Section~\ref{sec:data}: larger $\alpha$ causes upstream errors to induce stronger downstream errors. We vary only $\alpha$ while holding the graph topology and other settings fixed. Figures~\ref{fig:cascade_rho} and~\ref{fig:cascade_mse} evaluate spectral reduction and long-horizon error, respectively.

\begin{figure}[!ht]

\centering
\includegraphics[width=0.8\linewidth]{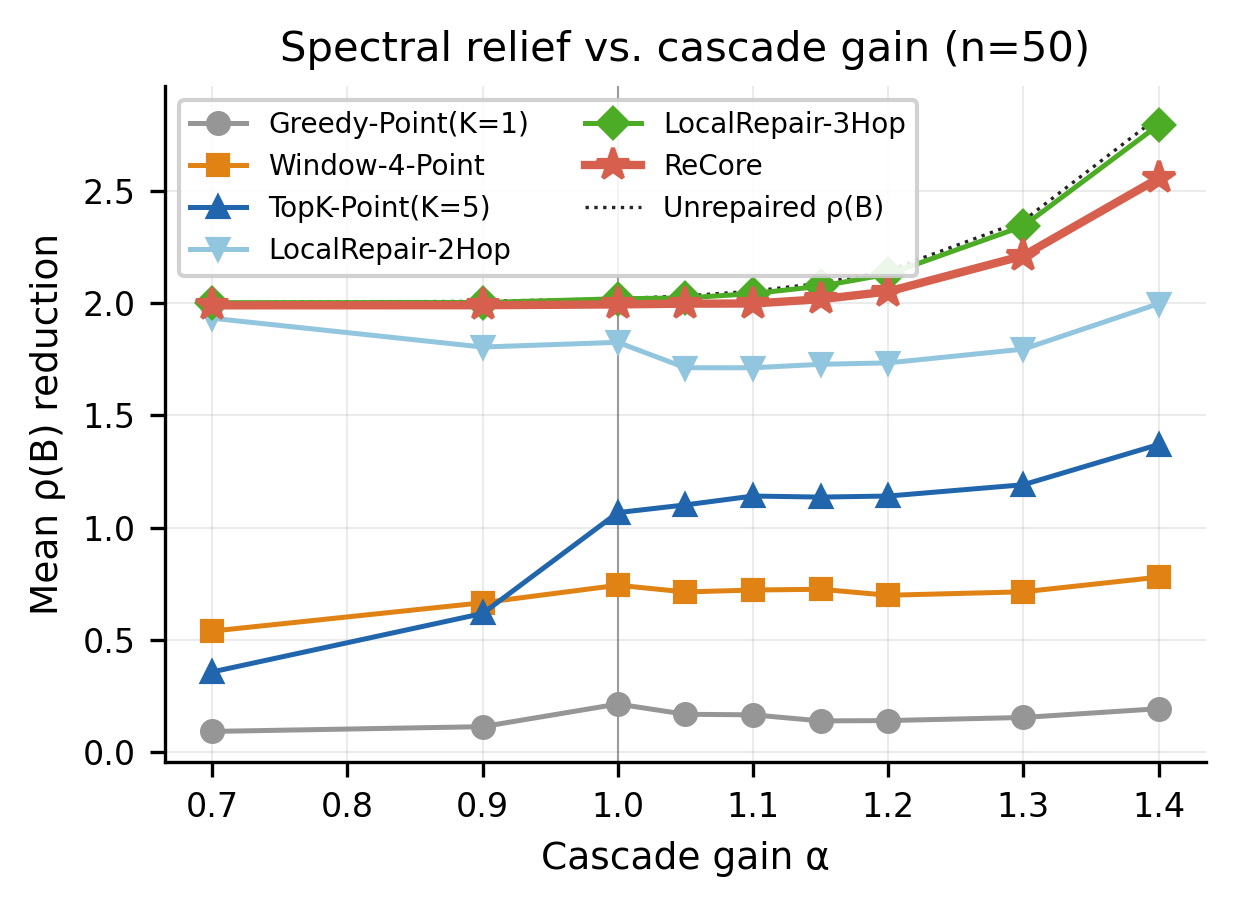}

\caption{Spectral reduction as error-propagation strength $\alpha$ increases ($n=50$).}
\label{fig:cascade_rho}
\Description{Residual spectral-radius reduction as cascade gain increases from 0.7 to 1.4. ReCore remains close to LocalRepair-3Hop and exceeds compact pointwise, window, and TopK baselines.}

\end{figure}

\begin{figure}[!ht]

\centering
\includegraphics[width=0.8\linewidth]{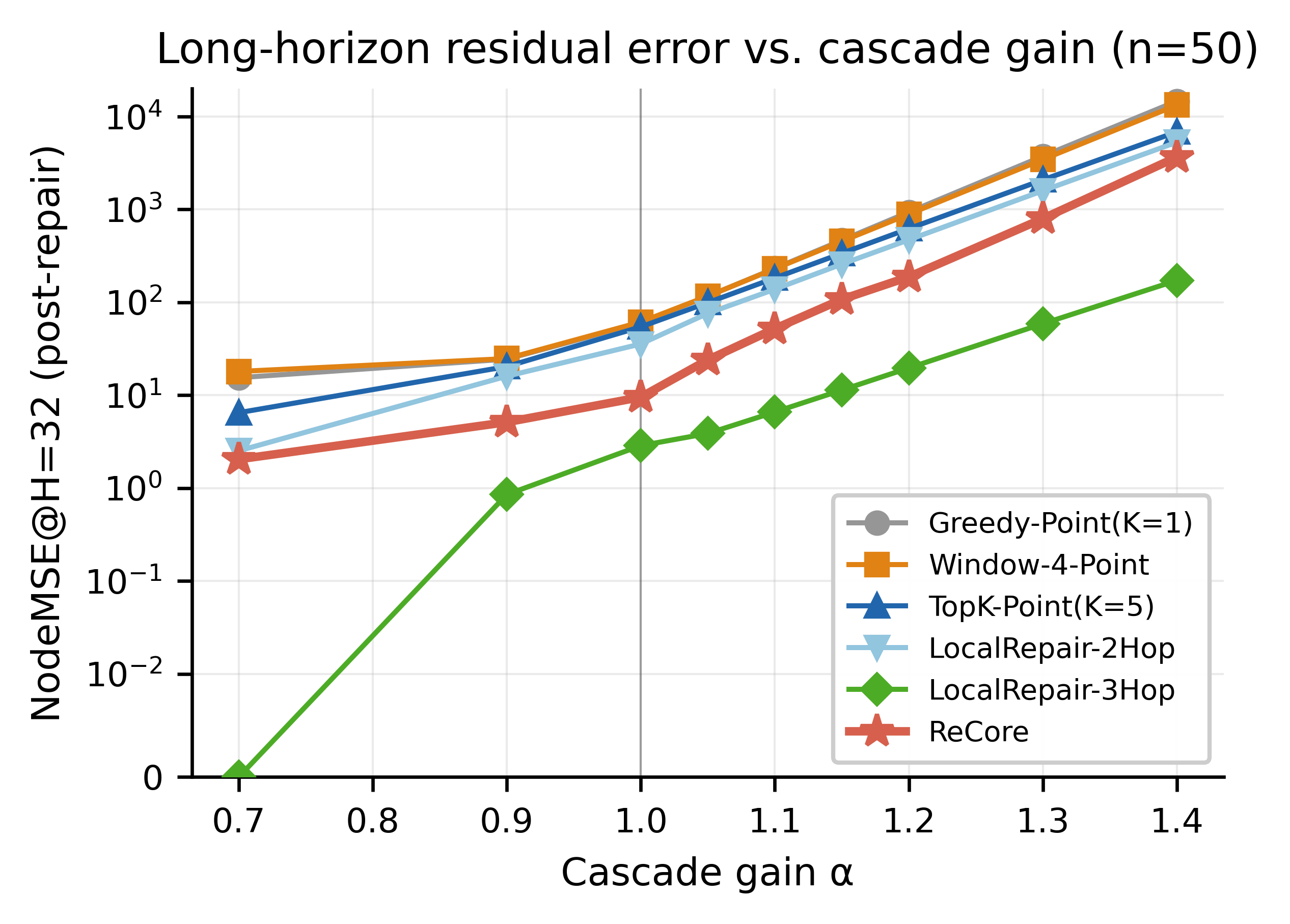}

\caption{NodeMSE@32 under increasing error-propagation strength $\alpha$.}
\label{fig:cascade_mse}
\Description{NodeMSE at horizon 32 as cascade gain increases. ReCore remains below compact engineering baselines, while LocalRepair-3Hop achieves the lowest error at high cascade gains.}

\end{figure}

Figure~\ref{fig:cascade_rho} shows that \method{} maintains at least $90\%$ spectral-repair efficiency across the full range, while Figure~\ref{fig:cascade_mse} shows consistently lower long-horizon error than compact engineering baselines. Thus, the advantage persists as error propagation becomes progressively stronger, although very large $\alpha$ leaves more residual error for any compact repair region.

\section{Conclusion}
\label{sec:conclusion}

We introduced \method{}, which repairs compact connected regions by suppressing residual spectral propagation rather than simply correcting observed errors. By targeting error-amplifying dependencies, \method{} improves long-horizon rollout stability while providing compact, informative context for LLM-based root-cause identification. Experiments across repair budgets, propagation strengths, graph topologies, and multiple LLMs show that \method{} achieves efficient spectral repair, stable long-horizon rollouts, and robust root-cause identification.
\bibliographystyle{ACM-Reference-Format}
\bibliography{main}




\clearpage
\appendix

\section{LLM Repair: Full Comparison}
\label{app:llmfull}

Table~\ref{tab:llm} and Figure~\ref{fig:llm_singleapi} provide the complete GPT-4o-mini comparison. \method{} achieves the highest exact root-cause recall while using substantially less context than full-graph repair.

\begin{figure}[!ht]
\centering
\includegraphics[width=\linewidth]{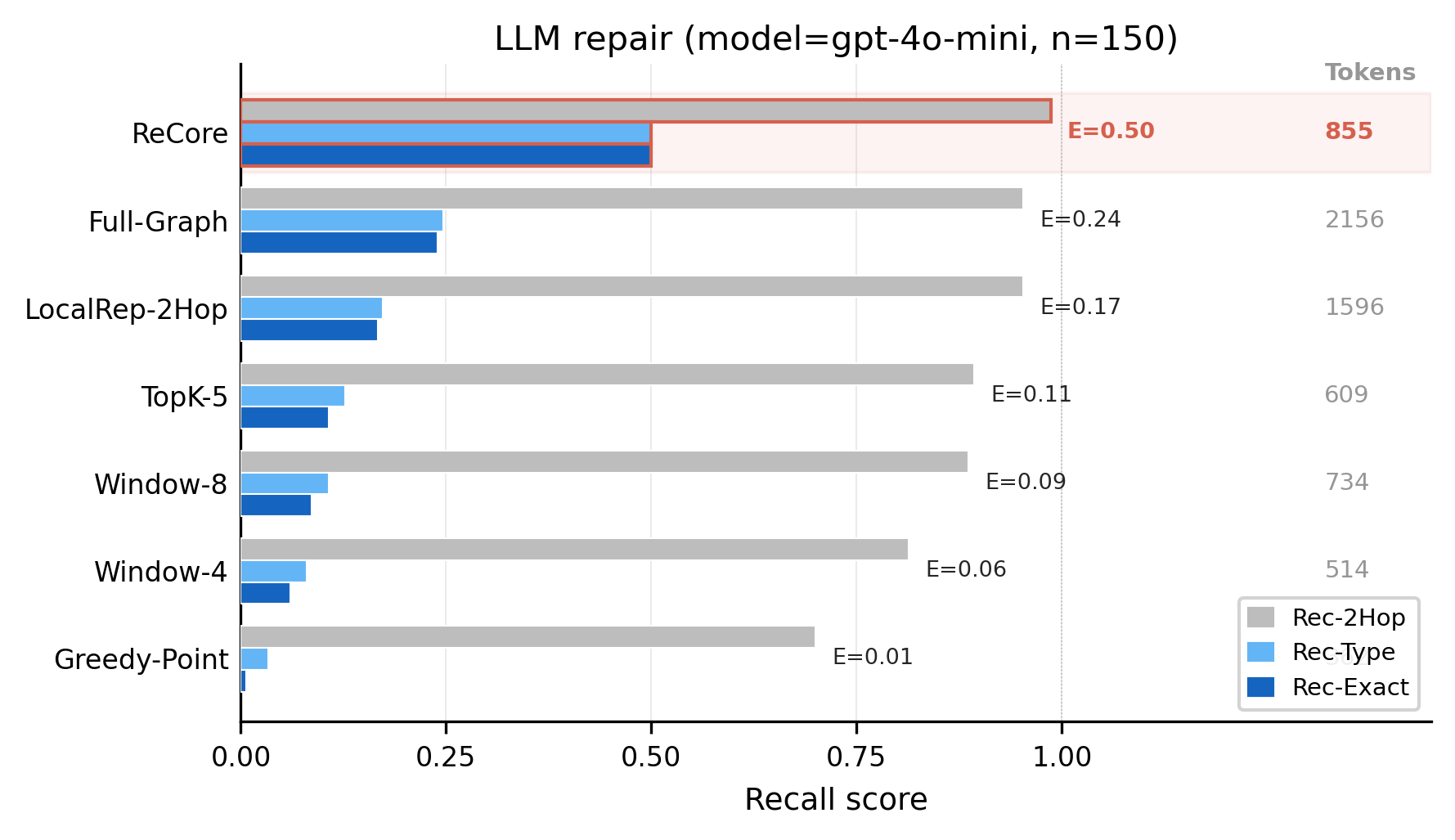}

\caption{LLM root-cause identification with GPT-4o-mini ($n=150$). Bars report Rec-Hop2, Rec-Type, and Rec-Exact, with mean token counts shown alongside. \method{} achieves the highest exact recall with an 855-token context.}
\label{fig:llm_singleapi}
\Description{Root-cause recall across repair-region selectors using GPT-4o-mini. ReCore achieves the highest exact recall while using a compact context.}

\end{figure}

\begin{table*}[!ht]
\centering\small\setlength{\tabcolsep}{3pt}
\caption{Full baseline comparison ($n=200$). \textbf{Rec} is the recovery rate, \textbf{Tokens} is the average token cost per instance, and \textbf{Tok/Rec} is the token cost per successful recovery, computed as $\text{Tokens}/\text{Rec}$; ``--'' indicates that no successful recovery occurred. \textbf{PD-red} and \textbf{DE-red\%} measure reductions in propagation depth and downstream error, respectively. Oracle-Region repairs the ground-truth corrupted region and is included only as an attribution reference. Among deployable methods, \method{} achieves the highest recovery rate while requiring $7.6\times$ fewer tokens per successful recovery than LLMRepair-Full-Plan. Best non-oracle values are shown in \textbf{bold}.}

\label{tab:spec_baselines}
\resizebox{0.6\textwidth}{!}{%
\begin{tabular}{lrrrrr}
\toprule
\textbf{Method} & \textbf{Rec}$\uparrow$ & \textbf{Tokens}$\downarrow$ & \textbf{Tok/Rec}$\downarrow$ & \textbf{PD-red}$\uparrow$ & \textbf{DE-red\%}$\uparrow$ \\
\midrule
\multicolumn{6}{l}{\textit{Spectral, oracle, and full-context repair}} \\
\method{} (ours)          & \textbf{0.625} & 1{,}424 & 2{,}278 & 5.4 & 60.0 \\
Oracle-Region             & 0.625 & 475 & 760 & 5.8 & 60.1 \\
LLMRepair-Full-Plan       & \textbf{0.625} & 10{,}857 & 17{,}371 & \textbf{5.8} & \textbf{62.6} \\
\midrule
\multicolumn{6}{l}{\textit{Engineering repair heuristics}} \\
TraceScan-Full-Point      & 0.620 & 21{,}714 & 35{,}023 & 0.8 & 47.0 \\
Top-B-Nodes (K=3)         & 0.620 & 600 & 968 & 0.8 & 47.0 \\
TraceScan-w4-Point        & 0.370 & 4{,}800 & 12{,}973 & 0.0 & 34.4 \\
Top-B-Edges (K=3)         & 0.320 & 600 & 1{,}875 & 0.0 & 28.7 \\
PageRank-Subgraph         & 0.140 & 518 & 3{,}700 & 0.4 & 23.2 \\
k-hop-k2                  & 0.135 & 550 & 4{,}074 & 0.0 & 26.0 \\
Uncertainty-Subgraph      & 0.045 & 317 & 7{,}044 & 0.0 & 6.4 \\
TraceScan-w2-Point        & 0.000 & 1{,}200 & -- & 0.0 & 16.4 \\
TraceScan-w1-Point        & 0.000 & 275 & -- & 0.0 & 7.1 \\
LastError-Point           & 0.000 & \textbf{50} & -- & 0.0 & 6.9 \\
FirstFailedCall-Point     & 0.000 & \textbf{50} & -- & 0.0 & 0.0 \\
\bottomrule
\end{tabular}}

\end{table*}

Table~\ref{tab:spec_baselines} provides the full engineering comparison on $n=200$ instances. The baselines represent practical repair strategies that an LLM corrector might use without an explicit spectral objective. TraceScan linearizes the execution trace and selects the full trace, a fixed window, or individual failed calls, following common debugging and failure-attribution heuristics~\citep{failure-attr-multi-agent,agentlens}. Top-B-Nodes and Top-B-Edges select the largest observed node or edge errors under a fixed budget, analogous to saliency-style top-$K$ repair. k-hop-k2/LocalRepair-2Hop performs graph-neighborhood repair~\citep{fault-subgraph,subgraph-anomaly}, while PageRank-Subgraph uses global graph influence~\citep{pagerank} and Uncertainty-Subgraph prioritizes uncertain nodes. LLMRepair-Full-Plan serializes the complete planning graph and applies the same LLM repair prompt as \method{}, providing a full-context analogue of LLM self-repair and reflection~\citep{reflexion2023,llm-repair-plan,lats2024}. Such full-context methods are substantially more expensive and can become difficult for LLMs to use reliably as context grows~\citep{liu2024lostmiddle}.

Recovery is the fraction of instances in which the root-cause region is identified and repaired. Tok/Rec denotes mean token cost per successful recovery when recovery is nonzero, and mean attempted token cost otherwise. PD-red and DE-red denote propagation-depth and downstream-error reduction, respectively.

\method{} reaches the highest non-oracle recovery (0.625) while requiring only 1{,}389 tokens per successful recovery, compared with 10{,}333 for LLMRepair-Full-Plan. TraceScan-Full-Point approaches the same recovery only at substantially higher token cost, while smaller pointwise and local heuristics frequently miss the region responsible for continued propagation. Consistent with Table~\ref{tab:main}, these results show that effective repair requires identifying a compact subgraph that suppresses residual dynamics rather than simply locating visible error.

\section{Proofs of Theoretical Results}
\label{app:proofs}

\begin{proof}[Proof of Theorem~\ref{thm:fixed_edge}]
Under Assumption~\ref{assump:fixed_lip}, with $L_X$ defined in Equation~\eqref{eq:fixed_lx}, let $a_k=e^X_k$. The one-step recursion in Theorem~\ref{thm:fixed_edge} becomes
\begin{equation}
a_{k+1}\le L_Xa_k+\epsilon_X.
\label{eq:app_fixed_rec}
\end{equation}
Iterating Equation~\eqref{eq:app_fixed_rec} gives
\begin{align}
a_k
&\le L_X\!\left(L_Xa_{k-2}+\epsilon_X\right)+\epsilon_X \nonumber\\
&\le L_X^2a_{k-2}+\epsilon_X(1+L_X) \nonumber\\
&\le L_X^ka_0+\epsilon_X\sum_{i=0}^{k-1}L_X^i.
\label{eq:app_fixed_unroll}
\end{align}
For $L_X\neq1$, the geometric sum in Equation~\eqref{eq:app_fixed_unroll} satisfies
\begin{equation}
\sum_{i=0}^{k-1}L_X^i=\frac{L_X^k-1}{L_X-1}.
\label{eq:app_fixed_geom}
\end{equation}
Substituting Equation~\eqref{eq:app_fixed_geom} into Equation~\eqref{eq:app_fixed_unroll} yields
\begin{equation}
e^X_k\le L_X^ke^X_0+\epsilon_X\frac{L_X^k-1}{L_X-1},
\end{equation}
which is exactly Equation~\eqref{eq:fixed_growth}. When $L_X=1$, Equation~\eqref{eq:app_fixed_unroll} instead gives $e^X_k\le e^X_0+k\epsilon_X$, completing the proof.
\end{proof}

\begin{proof}[Proof of Theorem~\ref{thm:coupled_operator}]
Under Assumption~\ref{assump:coupled}, the propagation operator is the nonnegative matrix defined in Equation~\eqref{eq:coupled_operator}. Its characteristic equation is
\begin{align}
0
&=\det(\lambda I-\Bop) \nonumber\\
&=(\lambda-L_X)(\lambda-M_A)-L_AM_X \nonumber\\
&=\lambda^2-(L_X+M_A)\lambda+(L_XM_A-L_AM_X).
\label{eq:app_coupled_char}
\end{align}
Solving Equation~\eqref{eq:app_coupled_char} gives
\begin{equation}
\lambda_{\pm}=\frac{L_X+M_A\pm\sqrt{(L_X-M_A)^2+4L_AM_X}}{2}.
\label{eq:app_coupled_roots}
\end{equation}
Because $\Bop$ is entrywise nonnegative, Perron--Frobenius implies that its spectral radius is the largest real eigenvalue. Hence, from Equation~\eqref{eq:app_coupled_roots},
\begin{equation}
\rho(\Bop)=\lambda_+=\frac{L_X+M_A+\sqrt{(L_X-M_A)^2+4L_AM_X}}{2},
\end{equation}
which is Equation~\eqref{eq:coupled_rho}.

If $L_AM_X>0$, then
\begin{equation}
\sqrt{(L_X-M_A)^2+4L_AM_X}>|L_X-M_A|.
\label{eq:app_coupled_strict}
\end{equation}
Combining Equations~\eqref{eq:app_coupled_roots} and~\eqref{eq:app_coupled_strict},
\begin{align}
\rho(\Bop)
&>\frac{L_X+M_A+|L_X-M_A|}{2} \nonumber\\
&=\max(L_X,M_A).
\end{align}
This proves the strict amplification claim in Theorem~\ref{thm:coupled_operator}.
\end{proof}

\begin{proof}[Proof of Theorem~\ref{thm:planning_regret}]
Under Assumption~\ref{assump:value_sensitivity}, for any policy $\pi$,
\begin{align}
|J(\pi)-\widehat J(\pi)|
&\le L_R\kappa\sum_{t=0}^{H-1}\gamma^t\|z_t\|+\epsilon_RH \nonumber\\
&\le L_R\kappa\epsilon\sum_{t=0}^{H-1}(\gamma\rhoB)^t+\epsilon_RH \nonumber\\
&=L_R\kappa\epsilon\phiH(\gamma,\rhoB)+\epsilon_RH.
\label{eq:app_value_error}
\end{align}
Define
\begin{equation}
C_H=L_R\kappa\epsilon\phiH(\gamma,\rhoB)+\epsilon_RH.
\label{eq:app_ch}
\end{equation}
Equation~\eqref{eq:app_value_error} therefore gives $|J(\pi)-\widehat J(\pi)|\le C_H$ uniformly for both $\pi^*$ and $\hat\pi$. Since $\hat\pi$ maximizes the learned-model value,
\begin{equation}
\widehat J(\pi^*)-\widehat J(\hat\pi)\le0.
\label{eq:app_planning_opt}
\end{equation}
Using Equations~\eqref{eq:app_value_error}--\eqref{eq:app_planning_opt},
\begin{align}
J(\pi^*)-J(\hat\pi)
&=J(\pi^*)-\widehat J(\pi^*)+\widehat J(\pi^*)-\widehat J(\hat\pi)+\widehat J(\hat\pi)-J(\hat\pi) \nonumber\\
&\le |J(\pi^*)-\widehat J(\pi^*)|+|\widehat J(\hat\pi)-J(\hat\pi)| \nonumber\\
&\le 2C_H \nonumber\\
&=2L_R\kappa\epsilon\phiH(\gamma,\rhoB)+2\epsilon_RH.
\end{align}
This is exactly Equation~\eqref{eq:regret}, proving Theorem~\ref{thm:planning_regret}.
\end{proof}

\section{Hyperparameters and Implementation}
\label{app:hyper}

\begin{center}\small
\begin{tabular}{ll}
\toprule
\textbf{Parameter} & \textbf{Value} \\
\midrule
$H$ (spectral walk depth) & 4 \\
$w$ (weight-norm proxy) & 0.9 \\
$K_{\max}$ (maximum region size) & 20 \\
$k_0$ (initial seeds) & 6 \\
$\lambda_1$ (error weight) & 1.2 \\
$\lambda_2$ ($\rho$-relief weight) & 1.5 \\
$\lambda_3$ (cost penalty) & 0.1 \\
Merge threshold $\tau$ & 0.5 \\
$\gamma$ (planning discount) & 0.95 \\
Cascade gain $\alpha$ & 1.1 \\
Cascade noise $\sigma$ & 0.05 \\
\bottomrule
\end{tabular}
\end{center}

\paragraph{Operator estimation.}
Because the failure graphs do not expose trained GWM weights, we estimate $L_X=w\rho(A_R)$, $L_A=0.3w\,d_{\rm in}\bar e$, $M_X=0.2w\,d_{\rm out}\bar e$, and $M_A=0.5w\,f_{\rm high}$, where $d_{\rm in}$ and $d_{\rm out}$ denote mean edge-type diversity, $\bar e$ is mean node error, and $f_{\rm high}$ is the fraction of edges connecting high-error nodes. The constants $(0.3,0.2,0.5)$ are calibration factors that keep the estimated components on comparable scales; they are not required by Theorem~\ref{thm:coupled_operator}. With a trained parametric GWM, these graph-based proxies can instead be replaced by direct estimates from the learned operators.

\section{Dataset Statistics}
\label{app:data}

\begin{center}\small
\begin{tabular}{@{}p{0.38\linewidth}@{\quad}p{0.45\linewidth}@{}}
\toprule
\textbf{Property} & \textbf{Value} \\
\midrule
Instances & 50\\
Nodes per graph & 22--30 (mean 25.7) \\
Node / edge types & 9 / 6 \\
State feature dimension & 8 \\
Failure types & 4 (drift, misfire, cascade, validator) \\
Root cause & planner (30\%) / executor (70\%) \\
Pre-repair $\rhoB$ & $1.97 \pm 0.30$ \\
Coupled-operator super-additivity & 100\% \\
GrowthSlope (unrepaired) & $+0.0026$ \\
\bottomrule
\end{tabular}
\end{center}

\section{LLM Prompt and Serialization Details}
\label{app:prompt}

The LLM receives only the selected repair region rather than the full planning graph. Nodes are serialized in topological order with their type, local state, observed error, confidence, and boundary edges; edges crossing the region boundary are marked as external context to preserve relevant dependencies without exposing the full graph.

\paragraph{Structured prompt.}
The LLM acts as a failure analyst and identifies the node that introduced the initial error from the serialized subgraph. It returns the root-cause node and a brief rationale in JSON format for deterministic evaluation; structured output is used only for parsing and does not affect region selection~\citep{openai2024structuredoutputs}.

\paragraph{Repair operator.}
In simulation, repair zeroes the selected node errors before re-simulating the rollout. In deployment, the same selection could trigger targeted re-execution of the corresponding planner, executor, validator, tool call, or memory update.

\end{document}